\documentclass[11pt]{article}

\usepackage[margin=1in]{geometry}
\usepackage[T1]{fontenc}
\usepackage{lmodern}
\usepackage{amsmath,amssymb,amsthm,mathtools}
\usepackage{microtype}
\usepackage{xcolor}
\usepackage[colorlinks=true,linkcolor=blue!55!black,urlcolor=blue!55!black,citecolor=blue!55!black]{hyperref}

\newtheorem{theorem}{Theorem}
\newtheorem{lemma}{Lemma}
\newtheorem{proposition}{Proposition}
\newtheorem{corollary}{Corollary}

\newcommand{\E}{\mathbb{E}}
\newcommand{\Pp}{\mathbb{P}}
\newcommand{\norm}[1]{\left\lVert #1\right\rVert}
\newcommand{\abs}[1]{\left\lvert #1\right\rvert}

\title{Logarithmic-Free Moment and Generalization Bounds for Uniformly
Stable Algorithms}
\author{
Thanh Nguyen-Cung\thanks{VinUniversity. Emails: \href{mailto:24thanh.nc@vinuni.edu.vn}{24thanh.nc@vinuni.edu.vn}, \href{mailto:binh.nt2@vinuni.edu.vn}{binh.nt2@vinuni.edu.vn}} 
\and 
Binh T. Nguyen\footnotemark[1]
}
\date{}

\begin{document}

\maketitle

\begin{abstract}
Uniform stability is a classical tool for controlling the generalization error of a learning algorithm. Bousquet, Klochkov, and Zhivotovskiy \cite{bkz2020} showed that the problem can be reduced to a moment inequality for a sum of weakly interacting functions of independent random variables. Their bound contains an additional factor $\log n$, and they asked whether this factor can be removed. We answer this upper-bound question affirmatively. More specifically, let $Z=(Z_1,\ldots,Z_n)$ have independent coordinates and let $g_i(Z)$ satisfy
\(\E[g_i(Z)\mid Z_{-i}]=0, \
\abs{\E[g_i(Z)\mid Z_i]}\le M,\) for every \(i = 1, \ldots, n,\)
where $Z_{-i}$ denotes all coordinates except $Z_i$. Assume additionally that changing any coordinate $Z_j$, $j\neq i$, changes $g_i$ by at most $\beta$, we prove that, for every $p\ge2$,
\[
\norm{\sum_{i=1}^n g_i(Z)}_{L_p}
\le 16pn\beta+M\sqrt{2pn}.
\]
This removes the $\log n$ factor from the previous bound and matches the lower bound of Bousquet, Klochkov, and Zhivotovskiy up to universal constants in the range covered by their construction.
Our proof first establishes the required estimate on the Rademacher cube, then transfers it to arbitrary product distributions by a two-copy randomization argument.
\end{abstract}

\section{Introduction}\label{sec:introduction}

Algorithmic stability provides a direct way to control the generalization error of a learning algorithm by measuring how sensitive its output is to changes in the training sample. Formally, let \(S=(Z_1,\ldots,Z_n)\sim P^n\) be an i.i.d.\ sample, let \(A_S\) denote the predictor returned by a learning algorithm \(A\), and let \(\ell\) be the loss. The population and empirical risks are
\[
R(A_S)=\E_{Z\sim P}\ell(A_S,Z), \qquad R_{\mathrm{emp}}(A_S)=\frac1n\sum_{i=1}^n\ell(A_S,Z_i),
\]
and their difference \(R(A_S)-R_{\mathrm{emp}}(A_S)\) is the generalization error. Following Bousquet and Elisseeff \cite{bousquet2002stability}, an algorithm is \(\gamma\)-uniformly stable if, whenever two samples \(S\) and \(S'\) differ in one coordinate,
\[
\sup_z\abs{\ell(A_S,z)-\ell(A_{S'},z)}\le\gamma.
\]
Uniform stability has become a standard tool for analyzing generalization, including for regularized learning methods \cite{bousquet2002stability} and stochastic gradient algorithms \cite{hardt2016train}. For bounded losses \(0\le\ell\le L\), the classical bounded-differences argument of Bousquet and Elisseeff gives, up to universal constants,
\[
\abs{R(A_S)-R_{\mathrm{emp}}(A_S)}
\lesssim
\left(\sqrt n\,\gamma+\frac{L}{\sqrt n}\right)\sqrt{\log\frac1\delta}
\]
with probability at least \(1-\delta\). Here, the factor \(\sqrt n\) multiplying the stability parameter is undesirable; for example, when \(\gamma\asymp n^{-1/2}\), this bound does not vanish with \(n\). Feldman and Vondr\'ak \cite{feldman2018generalization} showed that this loss is not intrinsic to uniform stability. They obtained the high-probability bound
\[
\abs{R(A_S)-R_{\mathrm{emp}}(A_S)}
\lesssim
\left(\sqrt{\gamma L}+\frac{L}{\sqrt n}\right)\sqrt{\log\frac1\delta},
\]
together with the second-moment estimate
\[
\E\left(R(A_S)-R_{\mathrm{emp}}(A_S)\right)^2\lesssim\gamma^2+\frac{L^2}{n}.
\]
The latter suggests that the natural dependence on stability should be linear in \(\gamma\). In subsequent work, Feldman and Vondr\'ak \cite{feldman2019highprob} obtained such a dependence, up to logarithmic factors:
\[
\abs{R(A_S)-R_{\mathrm{emp}}(A_S)}
\lesssim
\gamma\log n\log\frac{n}{\delta}
+
L\sqrt{\frac{\log(1/\delta)}{n}}.
\]
This left open whether the logarithmic dependence on the sample size is necessary. Bousquet, Klochkov, and Zhivotovskiy \cite{bkz2020} isolated this question through a product space moment problem. More specifically, let \(Z=(Z_1,\ldots,Z_n)\) have independent coordinates and let
\(
F(Z)=\sum_{i=1}^n g_i(Z),
\)
where, for every \(i\),
\[
\E[g_i(Z)\mid Z_{-i}]=0, \qquad \abs{\E[g_i(Z)\mid Z_i]}\le M,
\]
and changing \(Z_j\), for \(j\neq i\), changes \(g_i\) by at most \(\beta\). Here, \(M\) controls the contribution associated directly with the distinguished coordinate \(Z_i\), while \(\beta\) controls the interaction with the remaining coordinates. In the application to uniformly stable learning, these parameters correspond to the loss scale and the stability parameter, respectively. The authors then proved that, for every \(p\ge2\),
\begin{equation}\label{eq:intro-bkz}
\norm{\sum_{i=1}^n g_i(Z)}_{L_p}
\le
12\sqrt2\,pn\beta\lceil\log_2 n\rceil
+
4M\sqrt{pn}.
\end{equation}
Combined with their reduction from stable learning to this product space problem, this yields
\[
\abs{R(A_S)-R_{\mathrm{emp}}(A_S)}
\lesssim
\gamma\log n\log\frac1\delta
+
L\sqrt{\frac{\log(1/\delta)}{n}}
\]
with high probability. Thus, the remaining sample size dependence is concentrated in the factor \(\log n\) multiplying the interaction term in \eqref{eq:intro-bkz}. Their lower bound shows that this logarithmic factor is the only gap in the moment estimate. More precisely, Proposition~9 of \cite{bkz2020} constructs functions on the Rademacher cube such that, for an absolute constant \(\kappa>0\) and \(\kappa\le p\le n\),
\[
\norm{\sum_{i=1}^n g_i(Z)}_{L_p}
\gtrsim
pn\beta+M\sqrt{pn}.
\]
Hence, the dependence on \(p\), \(n\), \(\beta\), and \(M\) in \eqref{eq:intro-bkz} is optimal up to universal constants, except for the factor \(\log n\). Bousquet, Klochkov, and Zhivotovskiy therefore asked whether this factor can be removed for moments of order \(p>2\).

This question is specific to what follows from the centering and coordinate-sensitivity assumptions alone. Several related works obtain sharper guarantees under additional structure, including second-order difference conditions \cite{maurer2017second}, more general concentration assumptions \cite{kontorovich2014concentration}, data-dependent or local notions of stability \cite{foster2019hypothesis,deng2021locally}, Bernstein conditions \cite{klochkov2021stability}, and \(L_2\)-stability for randomized algorithms \cite{yuan2023l2}. On the lower-bound side, Liu and Lu \cite{liulu2020lower} constructed uniformly stable algorithms whose generalization error is of order
\(
\gamma+\frac{L}{\sqrt n}
\)
with constant probability. Thus, at constant confidence, the linear stability scale \(\gamma\) and the sampling scale \(L/\sqrt n\) are both unavoidable.

\paragraph{Our contribution.}
We remove the logarithmic factor in \eqref{eq:intro-bkz}. Under the same assumptions, we prove that, for every \(p\ge2\),
\begin{equation}\label{eq:intro-main}
\norm{\sum_{i=1}^n g_i(Z)}_{L_p}
\le
16pn\beta+M\sqrt{2pn}.
\end{equation}
Together with Proposition~9 of \cite{bkz2020}, this determines the moment growth up to universal constants for \(\kappa\le p\le n\). In particular, the interaction term has the optimal order \(pn\beta\) throughout the range covered by their lower-bound construction, with no logarithmic dependence on \(n\).

Applying the reduction of \cite{bkz2020} to a \(\gamma\)-uniformly stable algorithm with \(0\le\ell\le L\) gives, for every \(p\ge2\),
\[
\norm{R(A_S)-R_{\mathrm{emp}}(A_S)}_{L_p}
\le
33p\gamma
+
L\sqrt{\frac{2p}{n}}.
\]
Taking \(p\) of order \(\log(1/\delta)\) yields
\[
\abs{R(A_S)-R_{\mathrm{emp}}(A_S)}
\lesssim
\gamma\log\frac1\delta
+
L\sqrt{\frac{\log(1/\delta)}{n}}
\]
with probability at least \(1-\delta\). Thus, the stability contribution no longer contains any logarithmic dependence on the sample size. At constant confidence, the resulting rate is \(\gamma+L/\sqrt n\), matching the lower bound of Liu and Lu \cite{liulu2020lower} up to universal constants.

\paragraph{Proof outline.}
We first decompose
\[
g_i(Z)=\zeta_i(Z_i)+h_i(Z), \qquad \zeta_i(Z_i)=\E[g_i(Z)\mid Z_i].
\]
The variables \(\zeta_i(Z_i)\) are independent, centered, and bounded by \(M\), and therefore contribute at the standard \(M\sqrt{pn}\) scale. The residual terms satisfy the two centering conditions
\[
\E[h_i\mid Z_i]=0, \qquad \E[h_i\mid Z_{-i}]=0.
\]
We first prove the interaction bound for these doubly centered functions on the Rademacher cube. The main step is to control an exponential moment, which leads to a product whose expectation can be interpreted as the expected number of fixed points of a random map on the cube. We bound this quantity using its second factorial moment, obtaining an estimate that does not depend on the dimension. This allows us to avoid the partition argument over several scales that produces the factor $\log n$ in \cite{bkz2020}. We then extend the cube estimate to arbitrary product distributions by introducing an independent copy of the coordinates together with auxiliary Rademacher signs. This randomization creates an additional centering term, which we control separately using bounded differences. Combining these estimates gives the bound in \eqref{eq:intro-main}. 

\section{Notations and Preliminaries}\label{sec:preliminaries}

We begin by listing the notations used throughout the paper. For a positive integer $n$, we write $[n]=\{1,\ldots,n\}$. 
If $X$ is a real-valued random variable and $p\ge1$, we denote its $L_p$ norm by $\norm{X}_{L_p} = \left(\E\abs{X}^p\right)^{1/p}.$
Unless stated otherwise, all expectations are taken with respect to all random variables involved.

Our main results concern functions of independent coordinates. Let
$Z=(Z_1,\ldots,Z_n)$ be a random vector with independent coordinates, and
write
\(
Z_{-i}
=
(Z_1,\ldots,Z_{i-1},Z_{i+1},\ldots,Z_n)
\)
for the collection of all coordinates except $Z_i$. When an independent copy
of $Z$ is needed, we denote it by $Z'=(Z'_1,\ldots,Z'_n)$, where $Z'$ has the
same distribution as $Z$ and is independent of it.

We will occasionally condition on part of the randomness while taking
moments with respect to the remaining variables. If $X=X(U,V)$, we use the
notation
\(
\norm{X}_{L_p(U\mid V)}
=
\left(
\E_U\!\left[\abs{X}^p\mid V\right]
\right)^{1/p}
\)
for the conditional $L_p$ norm obtained by averaging over $U$ while keeping
$V$ fixed. Consequently, if
$\norm{X}_{L_p(U\mid V)}\le C$ almost surely in $V$, then
$\norm{X}_{L_p}\le C$.

We also need a convenient way to describe the effect of changing a single
coordinate. For a function $f$ on a product space, we say that its
$j$th-coordinate sensitivity is at most $\beta$ if
\(
\abs{f(z)-f(\widetilde z)}
\le \beta
\)
whenever $z$ and $\widetilde z$ differ only in coordinate $j$.

The first part of the proof is carried out on the Rademacher cube. Let
$\varepsilon=(\varepsilon_1,\ldots,\varepsilon_m)$ have independent
Rademacher coordinates. For a function
$u:\{-1,1\}^m\to\mathbb R$, define
\(
\E_i u = \E\!\left[u(\varepsilon)\mid\varepsilon_{-i}\right],
\ 
D_i u = u-\E_i u.
\)
For $x\in\{-1,1\}^m$, let $x^{(j)}$ denote the point obtained from $x$ by
flipping its $j$th coordinate. Since $\E_j$ averages over the two possible
values of that coordinate,
\begin{equation}\label{eq:discrete-derivative}
D_j u(x) = \frac{u(x)-u(x^{(j)})}{2}.
\end{equation}
It follows that if the $j$th-coordinate sensitivity of $u$ is at most
$\beta$, then $\norm{D_j u}_\infty \le \frac{\beta}{2},$ where
$\norm{u}_\infty = \max_{x\in\{-1,1\}^m}\abs{u(x)}.$

We also use the following standard inequalities. First, McDiarmid's
bounded-differences inequality \cite{mcdiarmid1989} states that if
$X=f(W_1,\ldots,W_N)$ is a function of independent random variables and the
$j$th-coordinate sensitivity of $f$ is at most $c_j$, then, for every
$t>0$,
\[
\Pp\{X-\E X\ge t\}
\le
\exp\left(
-\frac{2t^2}{\sum_{j=1}^N c_j^2}
\right).
\]
Second, the discrete Poincar\'e inequality on the Rademacher cube
\cite{ledoux2001} gives
\[
\operatorname{Var}(u)
\le
\sum_{j=1}^m
\E\abs{D_j u}^2.
\]

The following moment form of the bounded-differences inequality will also
be used repeatedly.

\begin{lemma}\label{lem:bd}
Let $X=f(W_1,\ldots,W_N)$, where $W_1,\ldots,W_N$ are independent.
Suppose that $\E X=0$ and that the $j$th-coordinate sensitivity of $f$ is
at most $c_j$. Then, for every $p\ge2$,
\[
\norm{X}_{L_p}
\le
\left(
\frac p2\sum_{j=1}^N c_j^2
\right)^{1/2}.
\]
\end{lemma}


\section{Main Result}\label{sec:main}

Let $Z=(Z_1,\ldots,Z_n)$ have independent coordinates, where $Z_i$ takes values in a measurable space $\mathcal Z_i$.

\begin{theorem}\label{thm:product}
Let $g_1,\ldots,g_n$ be real-valued measurable functions of $Z$. Assume that, for every $i\in[n]$,
\[
\abs{\E[g_i(Z)\mid Z_i]}\le M
\quad\text{a.s.},
\qquad
\E[g_i(Z)\mid Z_{-i}]=0
\quad\text{a.s.},
\]
and changing coordinate $Z_j$ changes $g_i$ by at most $\beta$ for every $j\neq i$.
Then, for every $p\ge2$, 
\[
\norm{\sum_{i=1}^n g_i(Z)}_{L_p} \le 16pn\beta+M\sqrt{2pn}.
\]
\end{theorem}


\subsection{Reduction to a Doubly Centered Interaction Term}\label{subsec:reduction}

We first split each $g_i$ into its conditional mean given $Z_i$ and a remaining doubly centered interaction term. Specifically, for each $i \in [n]$  define
\[
\zeta_i(Z_i)=\E[g_i(Z)\mid Z_i], \qquad h_i(Z)=g_i(Z)-\zeta_i(Z_i).
\]
Then, by the triangle inequality, 
\[
\left\|\sum_{i=1}^n g_i(Z) \right\|_{L_p} = \left \| \sum_{i=1}^n h_i(Z) + \sum_{i=1}^n \zeta_i(Z_i) \right\|_{L_p} \le \left\|\sum_{i=1}^n h_i(Z) \right\|_{L_p} + \left\| \sum_{i=1}^n \zeta_i(Z_i) \right\|_{L_p}. 
\]
We first control the second term. Since $\E[g_i\mid Z_{-i}]=0$, taking expectations gives $\E g_i=0$, and hence $\E \zeta_i=0$. Moreover, since each $\zeta_i(Z_i)$ depends only on $Z_i$, the independence of $Z_1,\ldots,Z_n$ implies that the variables $\zeta_i(Z_i)$ are independent. By assumption, they are also bounded in absolute value by $M$. Thus, changing a single coordinate $Z_j$ changes $\sum_i\zeta_i(Z_i)$ by at most $2M$, and Lemma~\ref{lem:bd} gives
\[
\norm{\sum_i \zeta_i(Z_i)}_{L_p} \le M\sqrt{2pn}.
\]
We now turn to the interaction terms $h_i$. By definition, $\E[h_i\mid Z_i]=0$. On the other hand, since $\zeta_i(Z_i)$ depends only on $Z_i$ and $Z_i$ is independent of $Z_{-i}$, we have $\E[\zeta_i(Z_i)\mid Z_{-i}]=\E\zeta_i=0$. Together with the assumption $\E[g_i\mid Z_{-i}]=0$, this yields
\[
\E[h_i\mid Z_i]=0, \qquad \E[h_i\mid Z_{-i}]=0.
\]
Thus, each $h_i$ satisfies both centering conditions. Finally, subtracting $\zeta_i(Z_i)$ does not affect the change caused by coordinate $Z_j$ for $j\neq i$, since $\zeta_i(Z_i)$ does not depend on $Z_j$. Therefore, it remains to bound $\|\sum_i h_i(Z)\|_{L_p}$ under these two centering conditions and the original coordinate sensitivity assumption.

\subsection{The Doubly Centered Estimate on the Rademacher Cube}\label{subsec:cube}

The two centering conditions obtained in Subsection~\ref{subsec:reduction} have a particularly rigid form on the Rademacher cube. This is the setting in which we prove the key interaction estimate.

Let $\varepsilon=(\varepsilon_1,\ldots,\varepsilon_m)$ be uniformly distributed on $\{-1,1\}^m$. By the notation of Section~\ref{sec:preliminaries}, the required cube estimate is the following proposition.
\begin{proposition}\label{prop:cube}
Let $h_i:\{-1,1\}^m\to\mathbb R$, $i\in[m]$, satisfy \(\E_i h_i=0\) and \( \E[h_i\mid\varepsilon_i]=0 \). Assume that for every $i\neq j$,
\[
\sup_{x\in\{-1,1\}^m} \abs{h_i(x)-h_i(x^{(j)})} \le\gamma.
\]
Then, for every $m\ge1$ and $p\ge2$, there is a universal constant $C$ such that
\[
\norm{\sum_{i=1}^m h_i}_{L_p} \le Cmp\gamma.
\]
\end{proposition}

We first record the following representation of $h_i$. 
\begin{lemma}\label{lem:representation}
For each $i\in[m]$, there is a function $q_i:\{-1,1\}^{m-1}\to\mathbb R$, independent of $\varepsilon_i$, such that
\[
h_i(\varepsilon) = \varepsilon_i q_i(\varepsilon_{-i}),
\qquad 
\E q_i=0.
\]
Moreover, for every $j\neq i$,
\begin{equation}\label{eq:qi-differences}
\norm{D_jq_i}_\infty \le \frac{\gamma}{2}, 
\qquad \text{and} \qquad \norm{q_i}_\infty \le \frac{m-1}{2}\gamma.
\end{equation}
\end{lemma} 

The exponential moment argument for Proposition~\ref{prop:cube} also uses the following product estimate.

\begin{lemma}\label{lem:product}
Suppose that $a_i:\{-1,1\}^m\to\mathbb R$ does not depend on coordinate $i$ and satisfies
\begin{equation}\label{eq:ai-assumptions}
\norm{a_i}_\infty \le c,
\qquad 
\norm{D_ja_i}_\infty \le \frac{c}{m}
\quad(j\neq i). 
\end{equation} 
for some $0 < c \le 1/(2e+2)$. Then there is a universal constant $C_0$ such that
\[
\E\prod_{i=1}^m \left(1+\varepsilon_i a_i(\varepsilon)\right) \le C_0.
\]
\end{lemma}

With these two ingredients stated separately, we now prove the cube estimate.

\begin{proof}[Proof of Proposition~\ref{prop:cube}] 
If $m = 1$ or $\gamma = 0$ the conclusion follows immediately from Lemma~\ref{lem:representation}. Hence assume $m \ge 2$ and $\gamma > 0$. 
Set $F := \sum_{i = 1}^m h_i$. For $x\geq0$, $t>0$, and every real $p>0$, maximizing $x^pe^{-tx}$ over $x\geq0$ gives 
\[ x^p\leq\left(\frac{p}{et}\right)^p e^{tx}.\]
Apply this with $x=\abs{F}$ and $t = \frac{c}{m\gamma}$ for a constant $c$ we choose later. We have for every real $p\geq2$, 
\[ \abs{F}^p \le \left(\frac{pm\gamma}{ec}\right)^p e^{t\abs{F}} \]
Then, taking the expectation yields
\[
  \norm{F}_{L_p} \leq \frac{pm\gamma}{ec} \left( \E e^{t\abs{F}} \right)^{\frac{1}{p}} \le \frac{pm\gamma}{ec} \left( \E e^{t\abs{F}} \right)^{\frac{1}{2}}
\]
Therefore, it remains to prove that $\E e^{t\abs{F}}$ is bounded by a universal constant. First, we use the elementary estimate $e^x \le (1+2x)^{1/2}e^{4x^2}\, \forall \abs x\le\frac14.$ Indeed, for $\psi(x) = \frac12\log(1+2x)+4x^2-x, \psi(0) =  0$ and taking the derivative gives 
\[
\frac{\mathrm{d}\psi}{\mathrm{d}x} = \frac{2x(3+8x)}{1+2x}. 
\]
Thus $\psi$ decreases on $[-1/4,0]$ and increases on $[0,1/4]$, so $\psi(x)\ge0$ on this interval. 

Applying this estimate for either $t_0 = t$ or $t_0=-t$ to each $t_0h_i$ gives
\begin{align} \label{eq:mgf-pointwise}
e^{t_0F} &\le \prod_{i=1}^m(1+2t_0h_i)^{1/2} \exp\left(4t^2\sum_{i=1}^mh_i^2 \right) \\
& =
\left(\prod_{i=1}^m(1+2t_0h_i)\right)^{1/2} e^{4t^2H^2} \qquad \left(\text{set } H := \left(\sum_{i = 1}^m h_i^2 \right)^{1/2} \right) \notag.
\end{align}
Taking the expectation and using the Cauchy--Schwarz inequality yields 
\begin{equation}\label{eq:mgf-cs}
\E e^{t_0F} \le \left( \E\prod_{i=1}^m(1+2t_0h_i)\right)^{1/2} \left(\E e^{8t^2H^2}\right)^{1/2}.
\end{equation}
We control the two factors in \eqref{eq:mgf-cs} separately. For the first one, set $a_i=2t_0q_i$. Since $h_i=\varepsilon_iq_i$, we have $1+2t_0h_i=1+\varepsilon_i a_i$, so Lemma~\ref{lem:product} applies once its hypotheses are verified below. It therefore remains to show that $\E e^{8t^2H^2}$ is bounded by a universal constant. 
To this end, we use the discrete Poincar\'e inequality on the Rademacher cube:
\[
\operatorname{Var}(u) \le \sum_{j=1}^m \E\abs{D_ju}^2.
\]
Since $\E q_i=0$, $D_iq_i=0$, and $\norm{D_jq_i}_\infty\le\gamma/2$ for $j\neq i$, we have 
\[
\E q_i^2 \le \sum_{j\neq i} \E\abs{D_jq_i}^2 \le \frac{m-1}{4}\gamma^2.
\]
Hence
\begin{equation} \label{eq:EH}
\E H \le (\E H^2)^{1/2} = \left(\sum_{i=1}^m\E h_i^2\right)^{1/2} =  \left(\sum_{i=1}^m\E q_i^2\right)^{1/2} \le \frac{\sqrt{m(m-1)}}{2}\gamma \le \frac{m\gamma}{2}.
\end{equation} 

For a flip of coordinate $j$, $q_j(\varepsilon)=q_j(\varepsilon^{(j)})$, while
$\abs{q_i(\varepsilon) - q_i(\varepsilon^{(j)})} \le \gamma
\, (i\neq j)$. The reverse triangle inequality in Euclidean space gives
\[
\abs{H(\varepsilon)-H(\varepsilon^{(j)})}
\le
\sqrt{m-1}\,\gamma.
\]
Therefore, McDiarmid's inequality yields
\begin{equation} \label{eq:Mc} 
\Pp\{H-\E H\ge \tau\} \le \exp\left( -\frac{2\tau^2}{m(m-1)\gamma^2} \right).
\end{equation}
Note that from \eqref{eq:EH}, $\E H + m\gamma s \le m\gamma \left( \frac12 + s \right)$. Hence, $ \left\{ H \ge m\gamma \left( \frac12 + s \right) \right\} \subseteq \left\{ H \ge \E H + m\gamma s \right\}$. It follows that with $\tau = m\gamma s$ in \eqref{eq:Mc},
\begin{align*}
\Pp \left\{ \frac{H}{m\gamma} \ge \frac{1}{2} + s \right\} 
& = \Pp \left\{ H \ge m\gamma\left( \frac12 + s \right) \right\} \\
& \le \Pp \{ H - \E H \ge m\gamma s \} \\
& \le \exp \left( -\frac{2ms^2}{m-1} \right) \\ 
& \le e^{-2s^2}. 
\end{align*}
Thus, setting $Y = \left( \frac{H}{m\gamma}-\frac12 \right)_+$, then $\Pp(Y>s)\le e^{-2s^2}$ and
\[
\left( \frac{H}{m\gamma} \right)^2 \le \left(Y+\frac12\right)^2 \le \frac12+2Y^2.
\]
For $0<\lambda<1$, using tail integration, we have 
\[
\E e^{2\lambda Y^2} = 1+ \int_0^\infty 4\lambda s e^{2\lambda s^2} \Pp(Y>s)\,ds
\le \frac{1}{1-\lambda}.
\]
Therefore, 
\[
\E \exp\left( \lambda\frac{H^2}{m^2\gamma^2}
\right) \le e^{\lambda/2} \E e^{2\lambda Y^2}
\le \frac{e^{\lambda/2}}{1-\lambda}.
\]
Now, choose $c = \frac{1}{2e+2}$. This choice gives $\|a_i\|_{\infty} \le 2t\frac{m-1}{2}\gamma \le c$ and for $j \neq i$, $\| D_ja_i\|_\infty \le 2t\frac{\gamma}{2} \le \frac{c}{m}$ which satisfy assumption \eqref{eq:ai-assumptions} in Lemma~\ref{lem:product}. Moreover, $\abs{t_0 h_i} = t\abs{q_i} \le \frac{c}{m\gamma}\cdot \frac{m-1}{2}\gamma = \frac{c(m-1)}{2m} < \frac{c}{2} < \frac14$ so that \eqref{eq:mgf-pointwise} applies. Finally, by the choice of $c$, $\lambda = 8c^2 < 1$ and hence $\E e^{8t^2H^2} \le \frac{e^{4c^2}}{1-8c^2}.$ In other words, $\E e^{8t^2H^2}$ is bounded by a universal constant as desired.   

Thus both factors on the right-hand side of \eqref{eq:mgf-cs} are bounded by universal constants, uniformly for $t_0=\pm t$. Hence $\E e^{tF}$ and $\E e^{-tF}$ are both uniformly bounded. Since $e^{t\abs{F}}\le e^{tF}+e^{-tF}$, we conclude that
\[
\E e^{t\abs{F}} \le \E e^{tF}+\E e^{-tF}
\]
is bounded by a universal constant, as required.
\end{proof} 

\subsection{Two-Copy Transfer to General Product Spaces}\label{subsec:transfer}

We now return to the doubly centered interaction term isolated in Subsection~\ref{subsec:reduction}. Pairing each original coordinate with an independent copy and introducing auxiliary Rademacher signs produces, conditionally on the pairs, exactly the cube structure controlled in Subsection~\ref{subsec:cube}.

\begin{proposition}\label{prop:transfer}
Let $h_1,\ldots,h_n$ be functions of arbitrary independent coordinates $Z_1,\ldots,Z_n$ satisfying
\begin{equation}\label{eq:double-center}
\E[h_i\mid Z_i]=0,
\qquad
\E[h_i\mid Z_{-i}]=0.
\end{equation}
Assume also that changing coordinate $Z_j$ changes $h_i$ by at most $\beta$ for every $j\neq i$.
Then, for every $p\ge2$,
\[
\norm{\sum_{i=1}^n h_i(Z)}_{L_p} \le 16pn\beta.
\]
\end{proposition}

\begin{proof}
Let $Z'=(Z'_1,\ldots,Z'_n)$ be an independent copy of $Z$ and set
\(
W_i=(Z_i^+,Z_i^-):=(Z_i,Z'_i), \ W=(W_1,\ldots, W_n).\)
We also introduce independent unbiased signs $\varepsilon_1,\ldots,\varepsilon_n$, independent of $W$, and use them to select one element from each pair $W_i$ by defining
\[
Z_i^\varepsilon=
\begin{cases}
Z_i^+,&\varepsilon_i=1,\\
Z_i^-,&\varepsilon_i=-1.
\end{cases}
\]
For each $i$, let $\varepsilon^{(i)}$ denote the sign vector obtained from $\varepsilon$ by flipping its $i$th coordinate. Conditionally on $W$, we then define
\[
Q_i(\varepsilon;W)
=
\frac12\left[
h_i(Z^\varepsilon)-h_i(Z^{\varepsilon^{(i)}})
\right].
\]

We first relate the original sum $\sum_i h_i(Z)$ to the randomized quantities $Q_i$. By the second centering identity in \eqref{eq:double-center},
\[ h_i(Z) = \E_{Z'_i} \left[ h_i(Z)-h_i(Z'_i,Z_{-i}) \mid Z \right]. \]
Therefore, Jensen's inequality yields
\begin{align}
\norm{\sum_{i=1}^nh_i(Z)}_{L_p} & \le \norm{ \sum_{i=1}^n \left[h_i(Z)-h_i(Z'_i,Z_{-i})\right]}_{L_p} \notag\\
& = 2\norm{\sum_{i=1}^nQ_i(\mathbf 1;W)}_{L_p}.
\label{eq:symmetrize}
\end{align}
The next observation allows us to replace the fixed orientation $\mathbf 1$ by random Rademacher signs. Fix a deterministic $\xi\in\{-1,1\}^n$ and swap the two entries of $W_j$ whenever $\xi_j=-1$. Since each pair $(Z_j,Z'_j)$ is exchangeable, this operation leaves the law of $W$ unchanged, while it transforms $Q_i(\mathbf1;W)$ into $Q_i(\xi;W)$. Averaging this identity over an independent uniform choice $\xi=\varepsilon$ therefore gives
\begin{equation}\label{eq:orientation}
\norm{\sum_iQ_i(\mathbf1;W)}_{L_p} = \norm{\sum_iQ_i(\varepsilon;W)}_{L_p}.
\end{equation}

We now examine the structure of $Q_i$ as a function of the Rademacher signs. For fixed $W$, flipping $\varepsilon_i$ reverses the sign of $Q_i$, so that $Q_i(\varepsilon^{(i)};W)=-Q_i(\varepsilon;W)$. Consequently, there exists a function $r_i(\varepsilon_{-i};W)$ such that
\[ Q_i(\varepsilon;W) = \varepsilon_i r_i(\varepsilon_{-i};W). \]
To impose the second centering condition required by Proposition~\ref{prop:cube}, we define
\[
b_i(W) = \E_{\varepsilon_{-i}} r_i(\varepsilon_{-i};W), \qquad \widetilde Q_i(\varepsilon;W) = Q_i(\varepsilon;W)-\varepsilon_i b_i(W).
\]
Conditionally on $W$, this definition ensures that
\[ \E[\widetilde Q_i\mid\varepsilon_{-i},W]=0, \qquad \E[\widetilde Q_i\mid\varepsilon_i,W]=0. \]
Moreover, for every $j\neq i$, the coordinate sensitivity assumption on $h_i$ implies
\[ \abs{Q_i(\varepsilon;W)-Q_i(\varepsilon^{(j)};W)} \le\beta.\]
Since the correction term $\varepsilon_i b_i(W)$ does not depend on $\varepsilon_j$ when $j\neq i$, subtracting it does not change this sensitivity bound. Proposition~\ref{prop:cube}, applied conditionally on $W$, therefore gives
\begin{equation}\label{eq:Qtilde}
\norm{\sum_{i=1}^n\widetilde Q_i}_{L_p(\varepsilon\mid W)} \le Cpn\beta.
\end{equation}

It remains to control the centering defect introduced by the terms $b_i(W)$. To make their dependence on $W$ explicit, fix a pair $W_i=(u,v)$ and define
\[
d_i(x_{-i};u,v)
=
\frac12
\left[
h_i(u,x_{-i})-h_i(v,x_{-i})
\right].
\]
With this notation, the definition of $b_i(W)$ can be rewritten as
\begin{equation}\label{eq:b-representation}
b_i(W)
=
\E_{\varepsilon_{-i}}
d_i(Z_{-i}^{\varepsilon};Z_i^+,Z_i^-).
\end{equation}
When we average over $W_{-i}$ and $\varepsilon_{-i}$, the random vector $Z_{-i}^{\varepsilon}$ has the same law as an independent draw from the original distribution of $Z_{-i}$. The first centering identity in \eqref{eq:double-center} therefore implies that $\E[b_i(W)\mid W_i]=0$. Moreover, if $j\neq i$, replacing the pair $W_j$ by another pair changes the right-hand side of \eqref{eq:b-representation} by at most $\beta$. Hence, applying Lemma~\ref{lem:bd} conditionally on $W_i$ gives
\[
\norm{b_i}_{L_p} \le \beta\sqrt{\frac{(n-1)p}{2}}.
\]

To aggregate these bounds over $i$, let \( B(W)=\left(\sum_i b_i(W)^2\right)^{1/2} \). Applying the triangle inequality in $L_{p/2}$ to $\sum_ib_i^2$ and then taking square roots yields
\[
\norm{B}_{L_p} \le \left(\sum_{i=1}^n\norm{b_i}_{L_p}^2\right)^{1/2}.
\]
Consequently,
\begin{equation}\label{eq:B-moment}
\norm{B}_{L_p} \le \beta\sqrt{\frac{n(n-1)p}{2}} \le n\beta\sqrt{\frac p2}.
\end{equation}

We next use this estimate to control the Rademacher sum generated by the centering defect. Conditionally on $W$, flipping the sign $\varepsilon_i$ changes $\sum_i\varepsilon_i b_i(W)$ by $2\abs{b_i(W)}$. A second application of Lemma~\ref{lem:bd} therefore gives
\[
\norm{\sum_i\varepsilon_i b_i(W)}_{L_p(\varepsilon\mid W)}
\le
\sqrt{2p}\,B(W).
\]
Combining this conditional estimate with \eqref{eq:B-moment} yields
\begin{equation}\label{eq:defect}
\norm{\sum_i\varepsilon_i b_i(W)}_{L_p} \le pn\beta.
\end{equation}
Hence, the decomposition of $Q_i$ into $\widetilde Q_i$ and the centering defect, together with \eqref{eq:Qtilde} and \eqref{eq:defect}, gives
\[
\norm{\sum_iQ_i(\varepsilon;W)}_{L_p} \le (C+1)pn\beta.
\]
Finally, Proposition~\ref{prop:cube}, with the choice
$c=1/(2e+2)$, gives $C<7$. Therefore, $2(C+1) < 16.$ Combining this estimate with \eqref{eq:symmetrize} and
\eqref{eq:orientation}, we obtain
\[
\norm{\sum_{i=1}^n h_i(Z)}_{L_p} \le 16pn\beta.
\]
This completes the proof.
\end{proof}

\begin{proof}[Proof of Theorem~\ref{thm:product}]
By the decomposition in Subsection~\ref{subsec:reduction} and Proposition~\ref{prop:transfer}, we have
\[
\begin{aligned}
\left\|\sum_{i=1}^n g_i(Z)\right\|_{L_p} &\le \left\|\sum_{i=1}^n h_i(Z)\right\|_{L_p} + \left\|\sum_{i=1}^n \zeta_i(Z_i)\right\|_{L_p} \\
&\le 16pn\beta+M\sqrt{2pn},
\end{aligned}
\]
as desired. 
\end{proof}

\section{Uniformly Stable Algorithms}\label{sec:stable}

We have the following corollary as a consequence of our result for learning algorithms. Let $A$ be $\gamma$-uniformly stable and suppose that the loss satisfies $0\le\ell\le L$.

\begin{corollary}\label{cor:stable}
For every $p\ge2$,
\[
\norm{R(A_S)-R_{\mathrm{emp}}(A_S)}_{L_p}
\le
33p\gamma+L\sqrt{\frac{2p}{n}}.
\]
Consequently, for every $\delta\in(0,1)$, with probability at least $1-\delta$,
\begin{equation}\label{eq:stable-tail}
\abs{R(A_S)-R_{\mathrm{emp}}(A_S)}
\le
e\left[
33\gamma\log\left(\frac{e^2}{\delta}\right)
+
L\sqrt{\frac{2}{n}\log\left(\frac{e^2}{\delta}\right)}
\right].
\end{equation}
\end{corollary}

\begin{proof}
The reduction of \cite[Lemma~7]{bkz2020} gives functions satisfying Theorem~\ref{thm:product} with $M=L$ and $\beta=2\gamma$, together with a deterministic comparison error at most $2\gamma n$ for the rescaled generalization gap. Applying Theorem~\ref{thm:product}, adding this error, and dividing by $n$ gives
\[
\norm{R(A_S)-R_{\mathrm{emp}}(A_S)}_{L_p} \le (32p+2)\gamma + L\sqrt{\frac{2p}{n}} \le 33p\gamma + L\sqrt{\frac{2p}{n}}.
\]
For the tail bound, set $q=\log(e^2/\delta)\ge2$. Markov's inequality gives
\[
\Pp\left\{ \abs{X}>e\norm{X}_{L_q} \right\} \le e^{-q} \le\delta.
\]
Applying this to the generalization error proves \eqref{eq:stable-tail}.
\end{proof}

\bibliographystyle{plain}
\bibliography{reference}

\newpage
\appendix 
\section{Proofs of Lemmas}

\subsection{Proof of Lemma~\ref{lem:bd}} \label{proof:lem-bd} 
\begin{proof}
Set $V=\sum_jc_j^2$ and using McDiarmid's inequality, we have
\[
\Pp\{\abs X>\tau\} \le 2\exp\left(-\frac{2\tau^2}{V}\right). \]
Therefore, 
\[
\E\abs X^p \le 2p\int_0^\infty t^{p-1}e^{-2t^2/V}\,\mathrm{d}t = 2\left(\frac V2\right)^{p/2}\Gamma\left(\frac p2+1\right) \le \left(\frac{Vp}{2}\right)^{p/2}.
\]
The last inequality follows from $\Gamma(x+1) \le x^x,\, \forall x \ge 1$. Taking $p$th roots proves the claim.
\end{proof} 

\subsection{Proof of Lemma~\ref{lem:representation}} \label{proof: lem-representation} 
\begin{proof}
Fix $i$ and $\varepsilon_{-i}$. The identity $\E_i h_i=0$ gives
\[
h_i(1,\varepsilon_{-i}) + h_i(-1,\varepsilon_{-i}) = 0.
\]
By defining $q_i(\varepsilon_{-i})=h_i(1,\varepsilon_{-i})$, we then have \(h_i(\varepsilon) = \varepsilon_iq_i(\varepsilon_{-i}).\)
Since $q_i$ is independent of $\varepsilon_i$, we obtain $0 = \E[h_i\mid\varepsilon_i] = \varepsilon_i\E q_i,$
so $\E q_i=0$. If $j\neq i$, flipping coordinate $j$ does not change $\varepsilon_i$, and therefore
\[
\abs{q_i(\varepsilon)-q_i(\varepsilon^{(j)})}
=
\abs{h_i(\varepsilon)-h_i(\varepsilon^{(j)})}
\le
\gamma.
\]
Equation \eqref{eq:discrete-derivative} gives the first part of \eqref{eq:qi-differences}.

For the supremum bound, fix $x\in\{-1,1\}^m$ and let $Y$ be an independent uniform point of the cube. Changing one coordinate at a time gives
\[
\abs{q_i(x)-q_i(Y)} \le \gamma\,d_H(x_{-i},Y_{-i}),
\]
where $d_H$ is Hamming distance. Since $\E q_i(Y)=0$,
\[
\abs{q_i(x)} \le \E\abs{q_i(x)-q_i(Y)} \le
\gamma\,\E d_H(x_{-i},Y_{-i}) = \frac{m-1}{2}\gamma.
\]
Taking the supremum over $x$ proves the second part of \eqref{eq:qi-differences}. 
\end{proof}

\subsection{Proof of Lemma~\ref{lem:product}} 
\begin{proof}
Let $U_1,\ldots,U_m$ be independent and uniform on $[-1,1]$. For $u=(u_1,\ldots,u_m)$, we define $\Phi_u:\{-1,1\}^m\to\{-1,1\}^m$ coordinatewise by
\[
\Phi_u(x)_i =
\begin{cases}
1,&u_i\le a_i(x),\\
-1,&u_i>a_i(x).
\end{cases}
\]
and let
\[
N(u) = \#\{x\in\{-1,1\}^m:\Phi_u(x)=x\}, \qquad A_x = \{\Phi_U(x)=x\}.
\]
Since $U_i$ are independent, we have 
\[
\Pp(A_x) = \prod_{i=1}^m \frac{1+x_i a_i(x)}{2}.
\]
and consequently,
\begin{equation}\label{eq:EN-product}
\E N = 2^{-m} \sum_{x\in\{-1,1\}^m} \prod_{i=1}^m (1+x_i a_i(x)) = \E \prod_{i=1}^m (1+\varepsilon_i a_i(\varepsilon)).
\end{equation} 
It remains to bound $\E N$.

Fix $x\in\{-1,1\}^m$. Let $S\subseteq[m]$, $\abs S=s$, and set $y=x^S$, meaning that all coordinates in $S$ are flipped. Conditional on $A_x$, the variables $U_i$ remain independent, and each is uniform on an interval of length at least $1-c$. For $i\in S$, the signs $x_i$ and $y_i$ are opposite. For both the $i$th coordinate conditions in $A_x$ and $A_y$ to hold, $U_i$ must lie between $a_i(x)$ and $a_i(y)$. Since $a_i$ does not depend on coordinate $i$, the points relevant to $a_i(x)$ and $a_i(y)$ differ in only the $s-1$ coordinates of $S\setminus\{i\}$. Therefore
\[
\abs{a_i(x)-a_i(y)}
\le
\frac{2c(s-1)}{m}.
\]
Ignoring the additional restrictions from coordinates outside $S$,
\begin{equation}\label{eq:conditional-Ay}
\Pp(A_y\mid A_x) \le \left( \frac{2c(s-1)}{(1-c)m} \right)^s.
\end{equation}
For $s=1$, the right-hand side is zero. Using \eqref{eq:conditional-Ay},
\begin{align*}
\E[N(N-1)]
&= \sum_x \Pp(A_x) \sum_{y\neq x} \Pp(A_y\mid A_x) \\
&\le K\,\E N,
\end{align*}
where
$K = \sum_{s=2}^m\binom{m}{s} \left( \frac{2c(s-1)}
{(1-c)m} \right)^s.$
Using inequality $\binom{m}{s}\le(em/s)^s$, we have
\[
K \le \sum_{s=2}^{\infty} \left( \frac{2ec}{1-c}\right)^s \left(1-\frac1s\right)^s.
\]
We have $r = \frac{2ec}{1-c} $, $ K^\star = \frac{1}{e}\cdot\frac{r^2}{1-r} < 2$ from $0 < c \le 1/(2e+2)$. Since $(1-1/s)^s\le e^{-1}$, it follows that 
\[
K \le \frac1e \sum_{s=2}^{\infty}r^s = \frac1e \cdot\frac{r^2}{1-r} < 2.
\]
Therefore, since $N \le 1 + \frac12N(N-1)$, we have 
\[ \E N \le 1+\frac12\E[N(N-1)] \le 1+\frac {K^\star}{2}\E N. \]
Since $1-K^\star/2 > 0,$ we get $\E N\le 2/(2-K^\star) = C_0$. Together with \eqref{eq:EN-product}, this proves the claim. 
\end{proof}

\end{document}